\documentclass[12pt]{article}
\usepackage{amsmath}
\usepackage{graphicx}
\usepackage{hyperref}
\pdfoutput=1
\usepackage{ dsfont }
\usepackage[latin1]{inputenc}
\usepackage{amsthm}
\usepackage{blkarray}
\usepackage{mathtools}
\usepackage{amsfonts}

\newtheorem*{theorem}{Theorem}
\theoremstyle{definition}
\newtheorem*{definition}{Definition}

\title{The Combinatorics of \textit{Salva Veritate} Principles}
\author{Norman E. Trushaev \footnote{The University of Vermont, Department of Mathematics}}
\date{January 13, 2022}

\begin{document}

\maketitle

\begin{abstract}
   Various concepts of grammatical compositionality arise in many theories of both natural and artificial languages, and often play a key role in accounts of the syntax-semantics interface. We propose that many instances of compositionality should entail non-trivial combinatorial claims about the expressive power of languages which satisfy these compositional properties. As an example, we present a formal analysis demonstrating that a particular class of languages which admit \textit{salva vertitate} substitutions - a property which we claim to be a particularly strong example of compositional principle - must also satisfy a very natural combinatorial constraint identified in this paper. 
\end{abstract}

\section*{Introduction}

This essay will present a formal framework and some preliminary results concerning the combinatorial properties of meaning-preserving substitutions of sentences. Following Quine (and ultimately Leibniz), we refer to such operations as \textit{salva veritate} substitutions. \footnote{It should be mentioned that Quine was primarily concerned with truth-value preserving substitutions (hence the \textit{veritate} bit), whereas we will be more broadly concerned with meaning preserving substitutions, in general, of which the truth-value preserving substitutions are a proper subset.}  

The relevant properties of \textit{salva veritate} substitutions are formalized in an abstract property that we introduce in the essay, and which we later call (SST). This property is itself a particular instance (and a very strong one) of a more general class of linguistic constraints that might be referred to as compositionality principles, where, by "compositionality principle", we simply mean any property that might be possessed by a language in virtue of the fact that it exhibits some form of grammatical compositionality. 

Our discussion will be divided into three parts. In Part 1, we briefly present some historical and conceptual background information which will be used to motivate our subsequent investigations. Part 2 is devoted to constructing our formal theory and presenting the main result. Finally, in Part 3, we conclude with a discussion of the preceding findings, and some remarks on directions for future work. 

\section{Historical and Conceptual Background}
By a "substitution" we mean any operation on strings of a language which replaces some sub-string with another. The simplest example would be any operation which replaces some word occurring in a sentence with some other word, but one can also consider substitutions of arbitrary syntactic constituents. Such substitution operations have played an important role in many debates in linguistics and philosophy, and have been particularly important in the analysis of intensional phenomenon, synonymy, and analyticity. 

The literature on these matters is extensive, and a full review is beyond the scope of this article. A good starting point, however, is Quine's classic "Two Dogmas of Empiricism" (\cite{quine}). Quine asks us to consider a typical analytic statement \[ (1) \indent \text{No bachelor is married.} \]
Quine suggests that a core feature of such statements is that they can be converted into logical truths by "putting synonyms for synonyms" (Quine, p. 23). In this particular case, if we substitute the phrase "unmarried man" for "bachelor", we obtain the logically tautological sentence \[(2) \indent \text{No unmarried man is married.}\]
In effect, analyticity of $(1)$ then rests on synonymy of the linguistic forms "unmarried man" and "bachelor", and the fact that $(2)$ is a logical truth. The explanatory burden now rests on finding a satisfactory account of linguistic synonymy. Regarding this matter, Quine points out (p. 27) that "A natural suggestion, deserving close examination, is that the synonymy of two linguistic forms consists simply in their interchangeability in all contexts without change of truth value - interchangeability, in Leibniz's phrase, \textit{salva veritate}."

This suggestion forms the starting point of our investigation. We, however, will not be concerned with examining the philosophical, semantic, or logical properties of \textit{salva veritate} substitutions, but rather their combinatorial properties, a line of investigation which remains largely unexplored. 

Although, perhaps not the conventional view, \textit{salve veritate} substitutions can be viewed as a specific instance of compositionality. Compositionality is a property exhibited by certain linguistic structures, whereby the meaning of any complex expression is determined by the meaning of it's component parts, and the manner in which these component parts are combined (\cite{szabo}). Compositionality has been an important subject of investigation in linguistic and logical scholarship since at least the work of Frege in the 19th century (\cite{frege}, \cite{janssen}, \cite{szabo}), and continues to be an active area of inquiry today (see, e.g. \cite{jacobson}, \cite{kracht})

\section{Formal Machinery}
For our purposes, we consider an interpreted language, by which we mean a set of strings over some alphabet, and a rule for assigning "meanings" to these strings. Formally, we define an interpreted language $\mathcal{L}$ to be an ordered tuple \[\mathcal{L}=(X, S, h, \mathcal{M}),\]
where $X$ denotes the alphabet, $S \subset X^*$ denotes the set of well-formed strings of the language (i.e. the things we can assign interpretations to), $h: S \rightarrow \mathcal{M}$ denotes the interpretation function, which assigns meanings to strings, and $\mathcal{M}$ is our set of meanings. In order to provide a fully general account, we make no assumptions about $\mathcal{M}$. Our "meanings", may therefore consist of anything, including truth-values, propositions, concepts, sentences of a meta-language, etc. For our purposes, we further assume that the alphabet $X$ is finite. 

We also introduce an additional constraint on $\mathcal{L}$. We call this constraint "substitutability of synonymous terms" (SST). Formally we have \begin{definition}[Substitutability of Synonymous Terms] 
Let $u, v \in S$ be well-formed strings satisfying the relation $h(u)=h(v)$. Then for any well-formed string $\alpha u \beta \in S$, the string $\alpha v \beta$ is also well-formed, and $h(\alpha u \beta) = h(\alpha v \beta)$.
\end{definition}
Informally, this condition states that if strings $u, v \in S$ have the same meaning, then the operation of substituting $v$ for $u$ is well-defined for all strings of our language $\mathcal{L}$, and the result of performing such an operation has no effect on meaning. It is important to note that such a condition places both a syntactic and a semantic constraint on $\mathcal{L}$. The syntactic component guarantees that replacing a substring $u$ that occurs in any well-formed string $w$, with a synonymous substring $v$, always produces another well-formed string $w'$. The semantic component guarantees that such substitutions have no impact on meaning, i.e. that synonymous constituents make the same semantic contribution to any linguistic context in which they occur. 

In addition to (SST), we will require another constraint, which we call "inductive constructibility" (IC). Formally we have
\begin{definition}[Inductive Constructibility]
Let $\mathcal{L}=(X, S, h, \mathcal{M})$ be an interpreted language. Then we say that $\mathcal{L}$ satisfies inductive constructibility iff any well-formed string $w\in S$ of length $n>1$ is equal to the concatenation $w=uv$ of two strings $u, v \in S$ which satisfy the relation $|w|= n = |u| + |v|$.  
\end{definition}

Informally, this condition states that any non-trivial string $w$ (i.e. a string of length $>1$) must be composed of smaller well-formed substrings. This prevents the existence of strings $w \in S$ that have non-trivial substrings which themselves are not well-formed. Hence, any non-trivial well-formed strings can be constructed by concatenating smaller well-formed strings. In particular, an inductive argument will immediately show that for $n>1$, any string $w\in S$ of length $|w|=n+1$ is equal to some string concatenation of the form $w=ua,$ where $|u|=n,$ and $x\in X$.

Before proceeding, we introduce some additional terminology that will be needed. Given subsets of well-formed strings $S_1, S_2 \subseteq S$, we say than $S_1$ is (strictly) more expressive that $S_2$ iff $h(S_2)$ is a (proper) subset of $h(S_1)$. Furthermore, for any subset $S' \subseteq S$ we define the expressive power of $S'$ to be equal to $|h(S_1)|$. We define the expressive power of an interpreted language $\mathcal{L}= (X, S, h, \mathcal{M})$ to be equal to the expressive power of $S$. Given a natural number $n \in \mathbb{N}$, we define the $n$-th generation of $S$ to be the set $\text{gen} (n) := \{w \in S : |w| \leq n \}$ of all well-formed strings of length $\leq n$. 

We are now ready for the main result: 

\begin{theorem}
An interpreted language $\mathcal{L}$ which satisfies (SST) and (IC) has infinite expressive power iff gen$(n+1)$ is strictly more expressive than gen$(n)$ for all $n \in \mathbb{N}$. 
\end{theorem}

\begin{proof}
Let $\mathcal{L}=(X, S, h, \mathcal{M})$ be an interpreted language which satisfies (SST) and (IC). We then have two directions to prove. \\
\indent ($\Rightarrow$) Let $\mathcal{L}$ have infinite expressive power, and suppose for contradiction that gen$(n+1)$ is not strictly more expressive gen$(n)$ for some $n\in \mathbb{N}$. Then $h(\text{gen}(n)) = h(\text{gen}(n+1))$, i.e. $\text{gen}(n)$ and $\text{gen}(n+1)$ have equal expressive power. In particular, we see that the expressive power of $\text{gen}(n+1)$ is finite and equal to \[|h(\text{gen}(n+1))| = |h(\text{gen}(n))| \leq |X|^n.\] 
\indent If $\mathcal{L}$ has infinite expressive power, then there exist meanings $m \in h(S)$ such that $m \notin h(\text{gen}(n+1))$. Now fix any such meaning $m$, and let $w\in S$ be a string of minimal length $|w|=k > n+1$ that is assigned meaning $h(w)=m$. By (IC), $w$ can be expressed in the form $w=uv$ for $u, v \in S$ of lengths $|u|=n+1$ and $|v|=k'=k-(n+1)$. Given that $u$ has length $n+1$, there exists a string $u' \in \text{gen}(n)$ such that $h(u)=h(u')$. By (SST), it follows that \[h(u'v)=h(uv)=m.\]
Now, since $|u'v| < |uv|$, this contradicts our assumption that $|w|=|uv|$ is the shortest possible length of any string that expresses meaning $h(w)=m$. Hence, the proper inclusion \[h(\text{gen}(n)) \subset h(\text{gen}(n+1))\] holds for all $n \in \mathbb{N}$. \\
\indent ($\Leftarrow$) Suppose that $\text{gen}(n+1)$ is strictly more expressive that $\text{gen}(n)$. Then for all $k\in \mathbb{N}$, there exists a meaning $m_k \in h(\text{gen}(k+1)-\text{gen}(k))$ that is expressible by a string of length $k+1$, but not by any string of length $\leq k$. Taking the union \[A:=\bigcup_{k=1}^\infty \{m_k\}, \] we then obtain an infinite collection of meanings $A \subseteq \mathcal{M}$ that is contained in $h(S),$ i.e. $A \subseteq h(S)$. Hence $\mathcal{L}$ has infinite expressive power.  

\end{proof}

\section{Conclusion}
The theorem of section 2 demonstrates a close connection between the admissible substitutions of a language, and its expressive power. In effect, this provides a practical demonstration of the fact that at least some non-trivial semantic information about a given language can be obtained purely by examining the meaning-preserving syntactic operations which are permitted by the language under consideration. Furthermore, this suggests that questions concerning the syntax-semantics interface, and the relationship between form and meaning, may have a non-trivial combinatorial component that is entirely independent of any particular semantic interpretation. 

The constraints placed on $\mathcal{L}$ in this paper, however, require additional refinement in order to be suitable for the description of most interesting examples of interpreted languages. Substitutability of synonymous terms (SST) and inductive constructibility (IC) are rather strong assumptions to make about a language. Such assumptions may in fact be perfectly innocuous and realistic assumptions for the analysis of many formal languages, but for many languages of interest, especially in the case natural languages, it seems plausible that (SST) will need to replaced - substituted, if you will - by a weaker constraint on substitutability. Well known intensional phenomenon, such as those identified and investigated by Russell, Kripke, Montague, Partee, (\cite{russell}, \cite{kripke}, \cite{montague}, \cite{partee})  and others, demonstrate conditions under which meaning is not preserved under substitutions of synonymous terms. In some cases, this is likely to require weakening our assumption regarding the preservation of well-formedness under substitution. In other cases, we will likely have to weaken our assumption that meaning remains identical under such substitutions. Many applications will likely require some combination of both. 

(IC) also appears to be quite problematic for many natural languages, for a variety of related reasons. Sentences of natural language are not, in general, composed of smaller sub-sentences. In the language of section 2, this means that the non-trivial well-formed strings of a natural language may not be composed of smaller well-formed substrings - which was a crucial property in our proof of the theorem presented in section 2. This seems to be closely related to the fact that most grammars of natural languages construct sentences not on the basis of concatenation of simpler formula, but rather on the basis of grammatical relations of either dependency (in the case of dependency grammars) or constituency (in the case of phrase structure grammars). 

Although significantly more complicated, such grammars do nevertheless impose relations of grammatical hierarchy and interdependence on the various syntactic forms of their respective languages. Under such relations, certain syntactic forms may be seen as more or less primitive in relation to others. This opens the door to extending our proof strategy used in our proof of the main result of section 2, to languages with more complicated grammatical relations. For our purposes, the crucial feature of (IC) is that it allowed us to relate certain properties of proper substrings to the properties of the well-formed strings in which they occur. In a similar fashion, we may hope to identify methods of relating the expressive properties of more primitive syntactic forms to the expressive properties of the more complex syntactic forms in which they occur. 

Complications aside, there is at least one area in which future investigations are likely to be considerably simpler. In an attempt to construct the most general possible theory, we have made no assumptions about the nature of $\mathcal{M}$. However, the semantics of most languages, whether they be formal or natural, generally allows for some additional structure on $\mathcal{M}$. In most cases, this will be some sort of logical, set theoretic, or algebraic structure. Whatever the case may be, this structure will furnish us with additional relations (i.e. constraints) between the various meanings and syntactic forms countenanced by the language. In general, the stronger these constraints, the more we can say about the combinatorial relations between the syntax and semantics of the language. 

Regarding (IC) and the structure of $\mathcal{M}$, several suggestions of Pietroski (\cite{pietroski}) appear to be a promising starting point for extending our methods to the analysis of natural language. In particular, Pietroski's analysis of how meanings compose in natural language suggests that many instances of compositionality may be reducible to logical conjunction. 

Future works on these topics, therefore has several areas to explore. First we can of course seek to prove more theorems about both (SST) and (IC), as we have done here. More generally, there are a variety of suggestions in the literature on compositionality, which make explicit claims about how the meanings of complex expressions of a language are related to the meanings their component parts. Many of these claims about compositionality are likely to entail specific combinatorial properties about their languages, which may then be identified using formal methods similar to those which we have employed in this paper. Finally, one might hope to obtain a deeper understanding of the relations between synonymy, intension, and compositionality by using these methods. These are rather distinct concepts, and yet they all seem to exhibit specific combinatorial properties. By identifying and relating the various combinatorial properties associated with these linguistic phenomenon, we may hope to thereby identify important relationships between synonymy, intension, and compositionality, and other related linguistic concepts.


\medskip

\begin{thebibliography}{}

\bibitem{frege} Frege, Gottlob [1884], (1980). \textit{The Foundations of Arithmetic}. 

\bibitem{jacobson} Jacobson, Pauline; Barker, Chris (2007). \textit{Direct Compositionality}. 

\bibitem{janssen} Janssen, Theo M. V. (2001). "Frege, Contextuality, and Compositionality". \textit{Journal of Logic, Language, and Information} 10 (1). Accessed via: 
https://www.jstor.org/stable/40180264

\bibitem{kracht} Kracht, Marcus (2011). \textit{Interpreted Languages and Compositionality}. 

\bibitem{kripke} Kripke, Saul (1979). "A Puzzle About Belief." Accessed via: \\
http://www.uvm.edu/~lderosse/courses/lang/Kripke(1979).pdf

\bibitem{montague} Montague, Richard (1973). "The Proper Treatment of Quantification in Ordinary English." Accessed via: \\ https://doi.org/http://www.cs.rhul.ac.uk/~zhaohui/montague73.pdf 

\bibitem{partee}  Partee, Barbara Hall (1970). "Opacity, coreference, and pronouns." \textit{Synthese} 21 (3-4). 359 - 385. Accessed via: \\
https://www.jstor.org/stable/20114733

\bibitem{pietroski} Pietroski, Paul (2018). \textit{Conjoining Meanings: Semantics Without Truth Values}. 

\bibitem{russell} Russell, Bertrand (1905). "On Denoting." Reproduced in \textit{Problems in the Philosophy of Language} (1969). Olshewsky, Thomas M.   

\bibitem{szabo} Szabo, Zoltan Gendler (2020). "Compositionality". \textit{The Stanford Encyclopedia of Philosophy"}. Accessed via:
https://plato.stanford.edu/entries/compositionality/

\bibitem{quine} Quine, Willard van Orman (1953). \textit{From a Logical Point of View: Nine Logico-Philosophical Essays}.


\end{thebibliography}
\end{document}